\documentclass[letterpaper]{article}
\usepackage[margin=1in,dvips]{geometry}
\usepackage{graphicx,psfrag,amsmath,amsthm,amssymb}
\usepackage{natbib}
\usepackage{algorithmic,algorithm}
\usepackage{url}

\newcommand{\B}{\ensuremath{\mathbf{B}}}
\newcommand{\C}{\ensuremath{\mathbf{C}}}

\newcommand{\LL}{\ensuremath{\mathbf{L}}}

\newcommand{\Y}{\ensuremath{\mathbf{Y}}}
\newcommand{\Z}{\ensuremath{\mathbf{Z}}}

\renewcommand{\c}{\ensuremath{\mathbf{c}}}

\newcommand{\g}{\ensuremath{\mathbf{g}}}

\newcommand{\uu}{\ensuremath{\mathbf{u}}}

\newcommand{\w}{\ensuremath{\mathbf{w}}}
\newcommand{\x}{\ensuremath{\mathbf{x}}}
\newcommand{\y}{\ensuremath{\mathbf{y}}}
\newcommand{\z}{\ensuremath{\mathbf{z}}}
\newcommand{\0}{\ensuremath{\mathbf{0}}}
\newcommand{\1}{\ensuremath{\mathbf{1}}}

\newcommand{\bbeta}{\ensuremath{\boldsymbol{\beta}}}

\newcommand{\bbR}{\ensuremath{\mathbb{R}}}

\newcommand{\calL}{\ensuremath{\mathcal{L}}}

\newcommand{\calO}{\ensuremath{\mathcal{O}}}

\newcommand{\norm}[1]{\left\lVert#1\right\rVert}

{%
\begin{list}{#1}{
\vspace{-\topsep}
\vspace{-\partopsep}
\setlength{\itemindent}{0cm}
\setlength{\rightmargin}{0cm}
\setlength{\listparindent}{0cm}
\settowidth{\labelwidth}{#1}
\setlength{\leftmargin}{\labelwidth}
\addtolength{\leftmargin}{\labelsep}
\setlength{\itemsep}{0cm}
}%
}%
{%
\end{list}
\vspace{-\topsep}
\vspace{-\partopsep}
}

{\begin{enumerate}%
}%
{\end{enumerate}}

\newcommand{\traceop}{\operatorname{tr}}
\newcommand{\trace}[1]{\ensuremath{\traceop\left(#1\right)}}

\theoremstyle{plain}

\newtheorem*{lemma*}{Lemma}

\newtheorem*{prop*}{Proposition}

\theoremstyle{definition}

\newtheorem*{defn*}{Definition}

\newtheorem*{exmp*}{Example}

\newtheorem*{conj*}{Conjecture}

\theoremstyle{remark}

\newtheorem*{rmk*}{Remark}

\bibpunct[, ]{(}{)}{;}{a}{,}{,} 

\newtheorem{theorem}{Theorem}

\title{Projection onto the probability simplex: \\ An efficient algorithm with a simple proof, and an application}
\author{Weiran Wang \hspace{3em} Miguel \'A. Carreira-Perpi\~n\'an \\
  Electrical Engineering and Computer Science, University of California, Merced \\
  \texttt{\{wwang5,mcarreira-perpinan\}@ucmerced.edu}}
\date{September 3, 2013}

\begin{document}

\maketitle

\begin{abstract}
  We provide an elementary proof of a simple, efficient algorithm for computing the Euclidean projection of a point onto the probability simplex. We also show an application in Laplacian $K$-modes clustering.
\end{abstract}

\section{Projection onto the probability simplex}

Consider the problem of computing the Euclidean projection of a point $\y=[y_1,\dots,y_D]^\top \in\bbR^D$ onto the probability simplex, which is defined by the following optimization problem:
\begin{subequations}
  \label{e:problem}
  \begin{align}
    \min_{\x\in\bbR^D} &\quad \frac{1}{2} \norm{\x-\y}^2 \\
    \text{s.t.} &\quad \x^\top \1 = 1 \\
    &\quad \x\ge \0.
  \end{align}
\end{subequations}
This is a quadratic program and the objective function is strictly convex, so there is a unique solution which we denote by $\x=[x_1,\dots,x_D]^\top$ with a slight abuse of notation.

\section{Algorithm}

The following $\calO(D\log{D})$ algorithm finds the solution \x\ to the problem:
\begin{algorithm}[h!]
  \caption{Euclidean projection of a vector onto the probability simplex.}
  \label{alg:proj}
  \renewcommand{\algorithmicrequire}{\textbf{Input:}}
  \renewcommand{\algorithmicensure}{\textbf{Output:}}
  \begin{algorithmic}
    \REQUIRE $\y \in\bbR^D$
    \STATE Sort $\y$ into \uu: $u_1 \ge u_2 \ge \dots \ge u_D$
    \STATE Find $\rho = \max\{1 \le j\le D\mathpunct{:}\ u_j+\frac{1}{j} (1-\sum_{i=1}^j u_i )>0\}$
    \STATE Define $\lambda = \frac{1}{\rho} (1 - \sum_{i=1}^\rho u_i )$
    \ENSURE \x\ s.t.\ $x_i = \max\{y_i+\lambda,0\}$, $i=1,\dots,D$.
  \end{algorithmic}
\end{algorithm}

\noindent
The complexity of the algorithm is dominated by the cost of sorting the components of $\y$. The algorithm is not iterative and identifies the active set exactly after at most $D$ steps. It can be easily implemented (see section~\ref{s:matlab}).

The algorithm has the following geometric interpretation. The solution can be written as $x_i = \max\{y_i+\lambda,0\}$, $i=1,\dots,D$, where $\lambda$ is chosen such that $\sum_{i=1}^D {x_i} = 1$. Place the values $y_1,\dots,y_D$ as points on the X axis. Then the solution is given by a rigid shift of the points such that the points to the right of the Y axis sum to $1$.

The pseudocode above appears in \citet{Duchi_08a}, although earlier papers \citep{Brucker84a,PardalKovoor90a} solved the problem in greater generality\footnote{\url{http://www.cs.berkeley.edu/~jduchi/projects/DuchiShSiCh08.html}}.

\paragraph{Other algorithms}

The problem~\eqref{e:problem} can be solved in many other ways, for example by particularizing QP algorithms such as active-set methods, gradient-projection methods or interior-point methods. It can also be solved by alternating projection onto the two constraints in a finite number of steps \citep{Michel86a}. Another way \citep[Exercise 4.1, solution available at \url{http://see.stanford.edu/materials/lsocoee364b/hw4sol.pdf}]{BoydVanden04a} is to construct a Lagrangian formulation by dualizing the equality constraint and then solve a 1D nonsmooth optimization over the Lagrange multiplier. Algorithm~\ref{alg:proj} has the advantage of being very simple, not iterative, and identifying exactly the active set at the solution after at most $D$ steps (each of cost $O(1)$) after sorting.

\section{A simple proof}

A proof of correctness of Algorithm~\ref{alg:proj} can be found in \citet{ShalevSinger06b} and \citet{ChenYe11a}, but we offer a simpler proof which involves only the KKT theorem.

We apply the standard KKT conditions for the problem \citep{NocedalWright06a}. The Lagrangian of the problem is
\begin{equation*}
  \calL(\x,\lambda,\bbeta) = \frac{1}{2} \norm{\x-\y}^2 - \lambda (\x^\top \1 - 1) - \bbeta^\top \x
\end{equation*}
where $\lambda$ and $\bbeta=[\beta_1,\dots,\beta_D]^\top$ are the Lagrange multipliers for the equality and inequality constraints, respectively. At the optimal solution $\x$ the following KKT conditions hold:
\begin{subequations}
  \label{e:KKT}
  \begin{align}
    x_i - y_i - \lambda - \beta_i&=0, \qquad i=1,\dots,D \\
    x_i &\ge 0, \qquad i=1,\dots,D \\
    \beta_i &\ge 0, \qquad i=1,\dots,D \\
    \label{e:compl}
    x_i \beta_i &= 0 , \qquad i=1,\dots,D \\
    \sum_{i=1}^D x_i &= 1. 
  \end{align}
\end{subequations}
From the complementarity condition~\eqref{e:compl}, it is clear that if $x_i>0$, we must have $\beta_i=0$ and $x_i = y_i + \lambda >0$; if $x_i=0$, we must have $\beta_i\ge 0$ and $x_i = y_i + \lambda + \beta_i=0$, whence $y_i+\lambda=-\beta_i\le 0$. Obviously, the components of the optimal solution \x\ that are zeros correspond to the smaller components of $\y$. Without loss of generality, we assume the components of $\y$ are sorted and $\x$ uses the same ordering , i.e.,
\begin{align*}
  y_1 \ge \dots \ge y_\rho \ge y_{\rho+1} \ge \dots \ge y_D ,\\
  x_1 \ge \dots \ge x_\rho > x_{\rho+1} = \dots = x_D,
\end{align*}
and that $x_1\ge \dots \ge x_{\rho} >0$, $x_{\rho+1}=\dots=x_D=0$. In other words, $\rho$ is the number of positive components in the solution \x. 
Now we apply the last condition and have
\begin{equation*}
  1 = \sum_{i=1}^D x_i = \sum_{i=1}^\rho x_i = \sum_{i=1}^\rho (y_i+\lambda)
\end{equation*}
which gives $\lambda = \frac{1}{\rho}(1- \sum_{i=1}^\rho y_i)$. Hence $\rho$ is the key to the solution. Once we know $\rho$ (there are only $D$ possible values of it), we can compute $\lambda$, and the optimal solution is obtained by just adding $\lambda$ to each component of $\y$ and thresholding as in the end of Algorithm~\ref{alg:proj}. (It is easy to check that this solution indeed satisfies all KKT conditions.) In the algorithm, we carry out the tests for $j=1,\dots,D$ if $t_j = y_j + \frac{1}{j}(1- \sum_{i=1}^j y_i)>0$. We now prove that the number of times this test turns out positive is exactly $\rho$. The following theorem is essentially Lemma 3 of \citet{ShalevSinger06b}.

\begin{theorem} Let $\rho$ be the number of positive components in the solution $\x$, then
  \begin{equation*}
    \rho = \textstyle\max\{ 1 \le j\le D\mathpunct{:}\ y_j+\frac{1}{j} (1 - \sum_{i=1}^j y_i)>0 \}.
  \end{equation*}
\end{theorem}
\begin{proof}
  Recall from the KKT conditions~\eqref{e:KKT} that $\lambda\rho = 1- \sum_{i=1}^\rho y_i$, $y_i +\lambda>0$ for $i=1,\dots,\rho$ and $y_i +\lambda \le 0$ for $i=\rho+1,\dots,D$. In the sequel, we show that for $j=1,\dots,D$, the test will continue to be positive until $j=\rho$ and then stay non-positive afterwards, i.e., $y_j+\frac{1}{j} (1 - \sum_{i=1}^j y_i) >0$ for $j\le \rho$, and $y_j+\frac{1}{j} (1 - \sum_{i=1}^j y_i) \le 0$ for $j > \rho$.
  \begin{itemize}
  \item[(i)] For $j=\rho$, we have 
    \begin{equation*}
      y_\rho+\frac{1}{\rho} \bigg(1 - \sum_{i=1}^\rho y_i\bigg) = y_\rho + \lambda = x_\rho >0.
    \end{equation*}
  \item[(ii)] For  $j<\rho$, we have
    \begin{multline*}
      y_j+\frac{1}{j} \bigg(1 - \sum_{i=1}^j y_i\bigg) = \frac{1}{j} \bigg(jy_j + 1- \sum_{i=1}^j y_i \bigg) = \frac{1}{j} \bigg(jy_j + \sum_{i=j+1}^\rho y_i + 1 - \sum_{i=1}^\rho y_i \bigg) = \frac{1}{j} \bigg(jy_j + \sum_{i=j+1}^\rho y_i + \rho\lambda \bigg) \\
      = \frac{1}{j} \bigg(j(y_j+\lambda) + \sum_{i=j+1}^\rho (y_i + \lambda )\bigg).
    \end{multline*}
    Since $y_i + \lambda >0$ for $i=j,\dots,\rho$, we have $y_j+\frac{1}{j} (1 - \sum_{i=1}^j y_i) >0$.
  \item[(iii)] For $j>\rho$, we have
    \begin{multline*}
      y_j+\frac{1}{j} \bigg(1 - \sum_{i=1}^j y_i\bigg) = \frac{1}{j} \bigg(jy_j + 1- \sum_{i=1}^j y_i \bigg) = \frac{1}{j} \bigg(jy_j + 1 - \sum_{i=1}^\rho y_i  - \sum_{i=\rho+1}^j y_i \bigg) = \frac{1}{j} \bigg(jy_j + \rho\lambda - \sum_{i=\rho+1}^j y_i \bigg) \\
      = \frac{1}{j} \bigg(\rho(y_j+\lambda) + \sum_{i=\rho+1}^j (y_j - y_i )\bigg).
    \end{multline*}
    Notice $y_j + \lambda \le 0$ for $j>\rho$, and $y_j\le y_i$ for $j\ge i$ since $\y$ is sorted, therefore $y_j+\frac{1}{j} (1 - \sum_{i=1}^j y_i) <0$.
  \end{itemize}
\end{proof}

\paragraph{Remarks} 

\begin{enumerate}
\item We denote $\lambda_j=\frac{1}{j}(1- \sum_{i=1}^j y_i)$. At the $j$-th test, $\lambda_j$ can be considered as a guess of the true $\lambda$  (indeed, $\lambda_\rho = \lambda$). If we use this guess to compute a tentative solution $\bar{\x}$ where $\bar{x}_i = \max\{y_i+\lambda_j,0\}$, then it is easy to see that $\bar{x}_i>0$ for $i=1,\dots,j$, and $\sum_{i=1}^j \bar{x}_i=1$. In other words, the first $j$ components of $\bar{\x}$ are positive and sum to $1$. If we find $\bar{x}_{j+1}=0$ (or $y_{j+1}+\lambda_j \le 0$), then we know we have found the optimal solution and $j=\rho$ because $\bar{\x}$ satisfies all KKT conditions.
\item To extend the algorithm to a simplex with a different scale, i.e., $\x^\top \1 = a$ for $a > 0$, replace the $1-\sum u_i$ terms with $a-\sum u_i$ in Algorithm~\ref{alg:proj}.
\end{enumerate}

\section{Matlab code}
\label{s:matlab}

The following vectorized Matlab code implements algorithm~\ref{alg:proj}. It projects each row vector in the $N\times D$ matrix \Y\ onto the probability simplex in $D$ dimensions.

\begin{verbatim}
function X = SimplexProj(Y)

[N,D] = size(Y);
X = sort(Y,2,'descend');
Xtmp = (cumsum(X,2)-1)*diag(sparse(1./(1:D)));
X = max(bsxfun(@minus,Y,Xtmp(sub2ind([N,D],(1:N)',sum(X>Xtmp,2)))),0);
\end{verbatim}

\section{An application: Laplacian $K$-modes clustering}
\label{s:appl}

Consider the \emph{Laplacian $K$-modes clustering algorithm} of \citet{WangCarreir13a} (which is an extension of the $K$-modes algorithm of \citealp{CarreirWang13a}). Given a dataset $\x_1,\dots,\x_N \in \bbR^D$ and suitably defined affinity values $w_{nm} \ge 0$ between each pair of points $\x_n$ and $\x_m$ (for example, Gaussian), we define the objective function
\begin{subequations}
  \label{e:Lap-kmodes}
  \begin{align}
    \min_{\Z,\C} &\quad \frac{\lambda}{2}\sum_{m=1}^N \sum_{n=1}^N w_{mn} \norm{\z_m-\z_n}^2 - \sum_{n=1}^N\sum_{k=1}^K z_{nk}  G \biggl( \norm{\frac{\x_n-\c_k}{\sigma}}^2 \biggr) \\
    \text{s.t.} & \quad \sum^K_{k=1}{z_{nk}} = 1, \text{ for } n=1,\dots,N \\
    & \quad z_{nk}\ge 0, \text{ for } n=1,\dots,N,\ k=1,\dots,K.
  \end{align}
\end{subequations}
\Z\ is a matrix of $N\times K$, with $\Z^{\top} = (\z_1,\dots,\z_N)$, where $\z_n = (z_{n1},\dots,z_{nK})^{\top}$ are the soft assignments of point $\x_n$ to clusters $1,\dots,K$, and $\c_1,\dots,\c_K \in \bbR^D$ are modes of the kernel density estimates defined for each cluster (where $G(\cdot^2)$ gives a Gaussian). The problem of projection on the simplex appears in the training problem, i.e., in finding a (local) minimum $(\Z,\C)$ of~\eqref{e:Lap-kmodes}, and in the out-of-sample problem, i.e., in assigning a new point to the clusters.

\paragraph{Training} The optimization of~\eqref{e:Lap-kmodes} is done by alternating minimization over \Z\ and \C. For fixed \C, the problem over \Z\ is a quadratic program of $NK$ variables:
\begin{subequations}
  \label{e:Lap-kmodes-training}
  \begin{align}
    \min_{\Z} & \quad \lambda \trace{\Z^\top \LL \Z}  - \trace{\B^\top \Z} \\
    \text{s.t.} & \quad \Z \1_K = \1_N \\
    & \quad \Z \ge \0,
  \end{align}
\end{subequations}
where $b_{nk} = G(\norm{(\x_n-\c_k)/\sigma}^2)$, $n=1,\dots,N$, $k=1,\dots,K$, and \LL\ is the graph Laplacian computed from the pairwise affinities $w_{mn}$. The problem is convex since \LL\ is positive semidefinite. One simple and quite efficient way to solve it is to use an (accelerated) gradient projection algorithm. The basic gradient projection algorithm iteratively takes a step in the negative gradient from the current point and projects the new point onto the constraint set. In our case, this projection is simple: it separates over each row of \Z\ (i.e., the soft assignment of each data point, $n=1,\dots,N$), and corresponds to a projection on the probability simplex of a $K$-dimensional row vector, i.e., the problem~\eqref{e:problem}.

\paragraph{Out-of-sample mapping} Given a new, test point $\x \in \bbR^D$, we wish to assign it to the clusters found during training. We are given the affinity vectors $\w = (w_n)$ and $\g = (g_k)$, where $w_n$ is the affinity between \x\ and $\x_n$, $n=1,\dots,N$, and $g_k = G(\norm{(\x-\c_k)/\sigma}^2)$, $k=1,\dots,K$, respectively. A natural and efficient way to define an out-of-sample mapping $\z(\x)$ is to solve a problem of the form~\eqref{e:Lap-kmodes} with a dataset consisting of the original training set augmented with \x, but keeping \Z\ and \C\ fixed to the values obtained during training (this avoids having to solve for all points again). Hence, the only free parameter is the assignment vector \z\ for the new point \x. After dropping constant terms, the optimization problem~\eqref{e:Lap-kmodes} reduces to the following quadratic program over $K$ variables:
\begin{subequations}
  \label{e:Lap-kmodes-oos}
  \begin{align}
    \min_{\z} & \quad \frac{1}{2} \norm{\z - (\bar{\z} + \gamma\g)}^2 \\
    \text{s.t.} & \quad \z^\top \1_K = 1 \\
    & \quad \z \ge \0
  \end{align}
\end{subequations}
where $\gamma = 1/2\lambda \sum_{n=1}^N{w_n}$ and
\begin{equation*}
  \bar{\z} = \frac{\Z^{\top} \w}{\w^{\top} \1} = \sum_{n=1}^N{\frac{w_n}{\sum_{n'=1}^N{w_{n'}}} \z_n}
\end{equation*}
is a weighted average of the training points' assignments, and so $\bar{\z} + \gamma\g$ is itself an average between this and the point-to-centroid affinities. Thus, the solution is the projection of the $K$-dimensional vector $\bar{\z}+\gamma\g$ onto the probability simplex, i.e., the problem~\eqref{e:problem} again.

\paragraph{LASS} In general, the optimization problem of eq.~\eqref{e:Lap-kmodes-training} is called \emph{Laplacian assignment model (LASS)} by \citet{CarreirWang13b}. Here, we consider $N$ items and $K$ categories and want to learn a soft assignment $z_{nk}$ of each item to each category, given an item-item graph Laplacian matrix \LL\ of $N \times N$ and an item-category similarity matrix \B\ of $N \times K$. The training is as in eq.~\eqref{e:Lap-kmodes-training} and the out-of-sample mapping as in eq.~\eqref{e:Lap-kmodes-oos}, and the problem of projection on the simplex appears in both cases.


\end{document}